\documentclass{article}
\usepackage[utf8]{inputenc}

\usepackage{amssymb}
\usepackage{amsmath}
\usepackage{amsthm}
\usepackage{latexsym}
\usepackage{graphicx}
\usepackage{color}
\usepackage{graphics}
\usepackage[round]{natbib}
\usepackage{cancel}
\usepackage{thm-restate}
\usepackage{mathtools}
\usepackage[mathscr]{euscript}
\usepackage{hyperref}

\usepackage{algorithm}
\usepackage{algorithmic}

\usepackage{fullpage}

\usepackage{MnSymbol}
\DeclareMathAlphabet\mathbb{U}{msb}{m}{n}
\usepackage{xpatch}

\DeclareMathOperator*{\E}{\mathbb E}

\DeclareMathOperator{\poly}{poly}
\DeclareMathOperator{\Reg}{\mathsf{Reg}}

\newtheorem{theorem}{Theorem}
\newtheorem{lemma}{Lemma}
\newtheorem{corollary}{Corollary}

\newtheorem{conjecture}{Conjecture}

\theoremstyle{definition}

\newcommand{\dest}{\mathrm{dest}}
\newcommand{\src}{\mathrm{src}}

\newcommand{\Bern}{\mathsf{Bern}}

\newcommand{\cA}{\mathcal{A}}
\newcommand{\cB}{\mathcal{B}}
\newcommand{\cC}{\mathcal{C}}

\newcommand{\cL}{\mathcal{L}}

\newcommand{\cS}{\mathcal{S}}

\newcommand{\cX}{\mathcal{X}}

\newcommand{\bell}{{\boldsymbol \ell}}
\newcommand{\blambda}{{\boldsymbol \lambda}}
\newcommand{\bmu}{{\boldsymbol \mu}}
\newcommand{\UB}{\mathsf{UB}}
\newcommand{\LB}{\mathsf{LB}}
\newcommand{\ILB}{\mathsf{ILB}}

\newcommand{\eps}{\varepsilon}

\newcommand{\ignore}[1]{}

\newcommand{\Ii}{\mathcal I}

\title{Adversarial Online Learning with Temporal Feedback Graphs}
\author{
Khashayar Gatmiry \thanks{MIT, {\tt gatmiry@mit.edu}} 
\and Jon Schneider \thanks{Google Research, {\tt jschnei@google.com}} 
}

\begin{document}

\maketitle

\begin{abstract}
We study a variant of prediction with expert advice where the learner's action at round $t$ is only allowed to depend on losses on a specific subset of the rounds (where the structure of which rounds' losses are visible at time $t$ is provided by a directed ``feedback graph'' known to the learner). We present a novel learning algorithm for this setting based on a strategy of partitioning the losses across sub-cliques of this graph. We complement this with a lower bound that is tight in many practical settings, and which we conjecture to be within a constant factor of optimal. For the important class of transitive feedback graphs, we prove that this algorithm is efficiently implementable and obtains the optimal regret bound (up to a universal constant). 
\end{abstract}

\section{Introduction}

Prediction with expert advice is one of the most fundamental problems in online learning. In its simplest form, a learner must choose from one of $K$ actions (possibly choosing a randomized mixture of actions) every round for $T$ rounds. An adversary then reveals a loss vector containing the loss for each action to the learner, the learner incurs their appropriate loss, and play proceeds to the next round. The goal of the learner in such settings is usually to minimize their regret: the gap between their total utility at the end of the game and the maximum utility they could have received if they played the best fixed action in hindsight. Notably, it is possible to construct algorithms for the learner which achieve regret sublinear in $T$ against any adversarially chosen sequence of losses.

Traditionally, a learner may use the entire history of losses up until round $t$ to decide their action at round $t$. In this paper, we investigate the question of what happens if we restrict the learner's action at time $t$ to depend on the losses in some subset of these rounds. Formally, we require the learner's (randomized) action at time $t$ to be a function of the losses in some subset of rounds $S_t$, for some fixed collection of subsets $\cS = \{S_t\}_{t=1}^{T}$ known to the algorithm a priori. We call this problem the problem of \emph{online learning with temporal feedback graphs}, in analogy with the use of feedback graphs to understand the value of partial information in online learning settings (e.g. \cite{mannor2011bandits}). 

In addition to being a mathematically natural extension of the classic problem of prediction with expert advice, this framework is general enough to model many problems of practical interest, including:

\begin{itemize}
    \item \textbf{Batched learning}: \sloppy{In batched learning, the time horizon $T$ is divided into ``batches'' of size $T_1, T_2, \dots, T_B$. Losses within a batch are only reported at the beginning of the next batch, so e.g. the action at a round $t$ in the $b$th batch must only depend on losses from the first $ b-1$ batches. }
    \item \textbf{Learning with delayed feedback}: In learning with delayed feedback, the loss at round $t$ is only reported to the learner at the end of round $t + \Delta$, after $\Delta$ rounds of delay (it is also possible to consider a round dependent delay $\Delta_t$, in which case this subsumes batched learning). 
    \item \textbf{Learning with bounded recall}: In learning with bounded recall, the learner is only allowed to use losses from the past $M$ rounds to decide their action. This captures the notion of bounded recall strategies for playing in repeated games introduced by \cite{aumann1989cooperation}.
    \item \textbf{``Learning from the future''}: Finally, nothing requires us to impose the constraint that $S_t$ only contains rounds before $t$: if the adversary fixes their sequence of loss vectors at the beginning of the game, then we can let the learner play a function of the losses in $S_t$ for any subset $S_t$ of $\{1, 2, \dots, T\}$. More practically, this can be used to model adversarial variants of prediction tasks used in the training of large language models: for example, in the task of ``masked language modelling'' (used in the training of BERT), the goal is to predict a masked token from a surrounding window of tokens \citep{devlin2018bert}. 
\end{itemize}

\subsection{Our results}

We investigate the problem of designing low-regret algorithms and proving regret lower bounds for the problem of online learning with temporal feedback graphs. 

\paragraph{Algorithms (Section \ref{sec:algs}).} On the topic of algorithms, we first remark that the seemingly natural algorithm of simply approximately best-responding to the set of losses you can see can actually be very far from optimal. This follows immediately from a result of \cite{schneider2022history}, who show that such algorithms can incur linear regret in $T$ in bounded recall settings (even when the recall window $M$ is large enough to obtain $O(\sqrt{T})$ regret by restarting every $M$ rounds). 

Instead, we propose a more sophisticated algorithm (Algorithm \ref{alg:ub_alg}) based on constructing a fractional decomposition of the directed feedback graph $\cS$ into ordered cliques, which we call ``orders'' for short. The key observation here is that the temporal feedback graph corresponding to the classic online learning setting (where you can observe all rounds in the past) consists of a single order of size $T$. Therefore, by partitioning the loss vector among the orders that are subgraphs of $\cS$, we can reuse the regret guarantees we have from the standard prediction with experts problem and get strong regret bounds for our algorithm.

We analyze the regret of this algorithm and show it is at most $O(\UB(\cS)\sqrt{\log K})$, where $\UB(\cS)$ is the optimal value of a specific convex program that we call the ``upper bound program'' (Theorem \ref{thm:ub_main}). Unfortunately, the size of this convex program (and in particular, the number of variables) is proportional to the number of maximal orders in $\cS$, which can be small for some graphs but in general is exponentially large in $T$. Moreover, actually running Algorithm \ref{alg:ub_alg} requires a feasible solution $\blambda$ to the upper bound program as input, and takes time proportional to the sparsity of $\blambda$. 

Luckily, by examining the dual convex program, we show that there exists an optimal solution $\blambda^*$ to the upper bound program supported on at most $T$ distinct orders (Lemma \ref{lem:basis}). This implies that, for a fixed $\cS$, there always exists an efficiently implementable learning algorithm achieving the optimum regret bound stated above (Corollary \ref{cor:efficient}), even if it may be hard to actually find this solution and construct this algorithm. 

\paragraph{Lower bounds (Section \ref{sec:lbs}).} We present two different methods for obtaining worst-case regret lower bounds for a fixed feedback graph $\cS$ (in the binary action case $K = 2$). Similarly as with $\UB(\cS)$, in both methods the lower bound is obtained by solving a convex program defined by the structure of $\cS$.

We obtain the first lower bound by extending the classic stochastic lower bound for prediction with expert advice, with the key difference that instead of sampling losses in an iid fashion, we let the loss in round $t$ have bias of magnitude $\eps_t$ for some set of $\eps_t \geq 0$. This results in a lower bound $\LB(\cS)$ on the regret of any algorithm, where $\LB(\cS)$ is given by the value of a polynomial-sized (and hence, efficiently solvable) convex program in the variables $\eps_t$ (Theorem \ref{thm:lbone}). Interestingly, the lower bound program is very similar to the dual of the upper bound program, which lets us immediately conclude that the upper bound $\UB(\cS)$ is near optimal for a wide variety of feedback graphs. In particular, we show that if the in-neighborhood $S_t$ of each round can be covered by at most $R$ orders of $\cS$, then $\UB(\cS) / \LB(\cS)$ is at most $O(\sqrt{R})$ (Theorem \ref{thm:path-cover}).

Unfortunately, there are many cases where the gap between $\LB(\cS)$ and $\UB(\cS)$ can be quite large (even polynomial in $T$). Inspired by this, we introduce a second lower bound which further generalizes the existing stochastic lower bound by allowing correlations between the losses in different rounds. In particular, we construct an adversary who, instead of sampling losses independently for each round, samples a random variable for each independent set of $\cS$, and builds the loss for round $t$ out of the random variables of all the independent sets containing $t$. By doing so, we introduce another convex program we call the \emph{independent set program}. Although this program is large and hard to solve in general (with number of variables equal to the number of independent sets in $\cS$), we prove that its value $\ILB(\cS)$ is a lower bound on the regret of any algorithm (Theorem \ref{thm:indepsets}) and conjecture that $\ILB(\cS)$ is within a constant factor of $\UB(\cS)$ (thus implying both this bound and our algorithm are within a constant factor of optimal).

\paragraph{Transitive feedback graphs (Section \ref{sec:transitive}).} Finally, we consider the interesting subclass of \emph{transitive feedback graphs}. These are feedback graphs $\cS$ where if $s \in S_t$ and $r \in S_s$, then $r \in S_t$; that is, if the learner can see loss $\ell_{s}$ at round $t$, they can also see all the losses $\ell_{r}$ that were visible at round $s$. Transitive feedback graphs that are a natural class of graphs that include many of our aforementioned applications, such as batched learning and learning with delayed feedback (both the round-dependent delay and round-independent delay case).

For transitive feedback graphs $\cS$, we show that we can indeed efficiently construct a sparse solution of the upper bound program, and hence efficiently construct and implement Algorithm \ref{alg:ub_alg} given $\cS$ (Theorem \ref{thm:transitive-main} in Section \ref{sec:transitive}). Doing so requires two technical insights: i. first, we show that we can construct an efficient separation oracle for the dual of the upper bound program via dynamic programming, ii. second, we show how we can use a solution $\bmu^*$ to the dual problem to reduce the problem of constructing a sparse solution of the upper bound problem to a flow-decomposition problem, which we can solve with standard techniques.

Finally, we show that for any transitive graph $\cS$ our upper bound $\UB(\cS)$ is within a constant factor of optimal (Theorem \ref{thm:bmlb} in Section \ref{sec:transitive}). To do so, we show how the adversary can take any feasible solution to the upper bound dual program and construct a correlated set of losses by observing a one-dimensional Brownian motion at various points in time. Specifically, the adversary associates each round $t$ with an interval $[p_t, q_t]$ of time, and sets the loss at round $t$ based on the sign of $L(q_t) - L(p_t)$ for a biased Brownian motion $L(\tau)$. By doing so, we effectively construct a feasible solution to our independent set program (which we can show has value within a constant factor of that of the value of our initial feasible solution). 

\subsection{Related work}

Since the seminal work of \cite{mannor2011bandits} introducing the problem of learning under partial information with feedback graphs, there has been an extensive body of literature extending and refining these results, including \citep{cortes2018online, cortes2019online, cortes2020online, balseiro2019contextual, alon2015online, cohen2016online, erez2021towards}. Although we take the name ``feedback graph'' from this line of work, the similarities between this line of work and the problem we study seem relatively superficial, limited mostly to the fact that both problems are parameterized by a directed graph and have regret bounds that depend on graph-theoretic properties of said graph. It would be interesting to show a stronger connection between these two questions.

Batched feedback has been studied fairly extensively in the online learning community, largely in the contexts of bandits and stochastic rewards \cite{perchet2016batched, gao2019batched, esfandiari2021regret}. The problem of learning with delayed feedback was first studied in the full-information setting by \cite{weinberger2002delayed}, and has since been studied in a variety of other learning settings \citep{mesterharm2005line, agarwal2011distributed, joulani2013online, quanrud2015online}. In particular, \cite{quanrud2015online} study the delayed feedback problem with adversarial delays -- our framework allows us to immediately recover optimal bounds for this setting given knowledge of the delays. Studying the behavior of agents with bounded recall is an active area in economics (see \cite{neyman1997cooperation} for a survey) which has recently been studied in the online learning setting by \cite{schneider2022history}. Finally, many sequence models in machine learning are constrained (often for efficiency / training reasons) so that the $t$th element of their output can only be based off of some pre-determined subset of the inputs. Often this subset is given by a context window (as in the bounded-recall setting), but it also can be structured in other interesting ways, as explored by \citep{beltagy2020longformer}.

\section{Model and Preliminaries}

\subsection{Online learning preliminaries}

We begin with an overview of the classic problem of prediction with expert advice. This can be viewed as a repeated game that takes place over $T$ rounds, where in each round the learner selects an action $x_{t} \in \cX$ and the adversary selects a loss $\ell_t \in \cL$. Here $\cX$ and $\cL$ refer to the action and loss set respectively; we will generally consider the setting $\cX = \Delta_{K}$ and $\cL = [0, 1]^K$ unless otherwise specified.

The learner selects their actions in accordance with some learning algorithm. Formally, a \emph{learning algorithm} $\cA$ is a collection of $t$ functions $A_t : \cL^{t-1} \rightarrow \cX$, with $A_t$ describing the action the learner takes at time $t$ as a function of the losses $\ell_1, \ell_2, \dots, \ell_{t-1}$. The regret of an algorithm $\cA$ on a sequence of losses $\bell = (\ell_1, \ell_2, \dots, \ell_T)$ is given by

\begin{equation}\label{eq:regret_def}
\Reg(\cA, \bell) = \sum_{t = 1}^{T} \langle x_t, \ell_t \rangle - \min_{x^* \in \cX} \sum_{t=1}^{T} \langle x^*, \ell_t\rangle,
\end{equation}

\noindent
where in \eqref{eq:regret_def}, $x_t = A_t(\ell_1, \ell_2, \dots, \ell_{t-1})$. We are often concerned with the worst-case regret of a learning algorithm, which we denote via $\Reg(\cA) = \max_{\bell \in \cL^{T}} \Reg(\cA, \bell)$.

One algorithm with asymptotically optimal regret for the above problem is the Hedge algorithm of \cite{freund1997decision}. The Hedge algorithm with learning rate $\eta > 0$ can be defined via:

\begin{equation}\label{eq:hedge}
A_t(\ell_1, \ell_2, \dots, \ell_{t-1})_i = \frac{\exp\left(-\eta\sum_{s=1}^{t-1}\ell_{s, i}\right)}{\sum_{j=1}^{K}\exp\left(-\eta\sum_{s=1}^{t-1}\ell_{s, j}\right)}
\end{equation}

It is possible to show that the worst-case regret of Hedge is at most $O(\sqrt{T \log K})$. We will need the following slightly finer-grained bound on the regret of Hedge.

\begin{lemma}\label{lem:loss-dependent-regret}
Let $\bell = (\ell_1, \ell_2, \dots, \ell_T)$ be a sequence of losses such that each $\ell_t \in [0, \lambda_t]^K$. If we let $\cA$ be the Hedge algorithm initialized with learning rate $\eta = O\left(\sqrt{(\log K)/\sum_{t=1}^T \lambda_t^2}\right)$, then

$$\Reg(\cA, \bell) \leq 2\sqrt{\left(\sum_{t=1}^T \lambda_t^2\right)\log K}.$$
\end{lemma}
\begin{proof}
    This follows directly from Theorem 1.5 in \cite{hazan2016introduction}.
\end{proof}

\subsection{Temporal feedback graphs}

We now describe the variant of learning from experts with temporal feedback graphs. A \emph{temporal feedback graph} $\cS$ is a collection of subsets $S_{t} \subseteq [T] \setminus \{t\}$ for each $t \in [T]$, where the set $S_{t}$ represents the set of losses visible to the learning algorithm at time $t$. Similar to our previous definition, an \emph{$\cS$-learning algorithm} $\cA$ is a collection of $t$ functions $A_{t} : \cL^{S_{t}} \rightarrow \cX$, with each $A_{t}$ describing the algorithm for mapping the visible losses to the action taken in round $t$. The regret $\Reg(\cA, \bell)$ of $\cA$ is defined identically as in \eqref{eq:regret_def}, with the only difference being that $x_t$ is now determined via $x_t = A_{t}(\ell_{S_{t}[1]}, \ell_{S_{t}[2]}, \dots, \ell_{S_{t}[|S_{t}|]})$ (where $(S_{t}[1], S_{t}[2], \dots, S_{t}[|S_{t}|])$ is some explicit enumeration of $S_t$). Our goal throughout this paper is to design efficient $\cS$-learning algorithms $\cA$ that minimize $\Reg(\cA)$.

In Section \ref{sec:transitive} we consider the special subclass of transitive feedback graphs $\cS$. We say a graph $\cS$ is \emph{transitive} if it is acyclic and has the property that if $s \in S_t$ and $r \in S_{s}$, then $r \in S_{t}$. That is, if the learner can see loss $\ell_{s}$ at round $t$, then they can also see all the losses they could see at round $s$. This class of graphs captures many natural applications (including the batched and delayed settings mentioned in the introduction).

\section{Algorithms}\label{sec:algs}

\subsection{A sub-optimal algorithm}

We begin our discussion of $\cS$-learning algorithms with a remark about a natural algorithm which turns out to be surprisingly sub-optimal. This algorithm simply does the following: at round $t$, play the strategy suggested by Hedge (with some learning rate $\eta$) when run on all losses visible at time $t$. That is, in round $t$, play the strategy defined via

$$A_t(\ell^{S_t})_i = \frac{\exp\left(-\eta\sum_{s\in S_t}\ell_{s, i}\right)}{\sum_{j=1}^{K}\exp\left(-\eta\sum_{s\in S_t}\ell_{s, j}\right)}$$

\noindent
for some choice of learning rate $\eta > 0$. One can also think of this strategy as approximately best responding to the average loss visible at time $t$. 

\cite{schneider2022history} show that in some bounded recall settings, this algorithm can incur regret that is linear in the time horizon $T$.

\begin{lemma}
Let $K = 2$, $M = T/10$ and let $\cS$ be the associated bounded recall feedback graph (where $S_{t} = \{t-M, t-M+1, \dots, t-1\}$). Then for any choice of learning rate $\eta$, the above algorithm $\cA$ has worst-case regret $\Reg(\cA) = \Omega(T)$. In contrast, there exists an algorithm $\cA'$ with $\Reg(\cA') = O(\sqrt{T})$.
\end{lemma}
\begin{proof}
See Theorem 2 of \cite{schneider2022history}. The sublinear regret algorithm $\cA'$ simply restarts Hedge every $M$ rounds.
\end{proof}

\subsection{A better algorithm}

In this section, we present an $\cS$-learning algorithm that does not have this detrimental trait, and indeed that we conjecture obtains within a constant factor of the optimal regret bound. As in the previous section, we will also use Hedge as a building block to construct this algorithm -- however, we will have to pay much more attention to the structural properties of the graph $\cS$. 

For a temporal feedback graph $\cS$, we say a sequence of rounds $t_1, t_2, \dots, t_w \in [T]$ forms an \emph{order} if for all $u < v$, $t_{u} \in S_{t_v}$. That is, every node later in the order can see the losses of all nodes earlier in the order. An order $C$ is \emph{maximal} if no super-sequence of $C$ forms an order. We will let $\cC = \{C_1, C_2, \dots, C_N\}$ denote the collection of all maximal orders in $\cS$.

Note that in the classic online learning problem, the entire temporal feedback graph $\cS$ is just a maximal order of size $T$. We can therefore think of Hedge (or any standard learning algorithm) as being a learning algorithm tuned specifically for feedback graphs that are orders. Inspired by this, we introduce an algorithm for the general case based on the idea of running multiplicative weights in parallel for every order in the graph and taking an optimal convex combination of these sub-algorithms.

To define this optimal convex combination (and characterize the resulting regret bound we get), we need to solve the following convex program. We will call this program the \emph{upper bound (primal) program}.

\begin{align}
\text{Minimize} \quad & \sum_{c=1}^{N} \sqrt{\sum_{t \in C_c} \lambda_{c, t}^2} \label{eq:ub_lp} \\
\text{Subject to} \quad &\sum_{c=1}^{N} \lambda_{c, t} = 1 \quad \text{ for all } t \in [T] \notag \\
& \lambda_{c, t} = 0 \quad \text{ if } t \not\in C_c \notag \\
& \lambda_{c, t} \geq 0 \quad \text{ for all } t \in [T], c \in [N]\notag
\end{align}

Given a set of variables $\lambda_{c, t}$ satisfying the upper bound program \eqref{eq:ub_lp}, we can define our $\cS$-learning algorithm as in Algorithm \ref{alg:ub_alg}. Intuitively, Algorithm \ref{alg:ub_alg} instantiates an instance of Hedge for every order in $\cC$. In round $t$, the algorithm partitions the loss of that round among the orders that pass through $t$, where the algorithm corresponding to the order $C_c$ receives a $\lambda_{c, t}$ fraction of the loss. 



    

\begin{algorithm}
\caption{Algorithm for solving online learning with temporal feedback graphs} \label{alg:ub_alg}
\begin{algorithmic}[1]
\STATE \textbf{Input: } A temporal feedback graph $\cS$ with time horizon $T$ and $K$ actions, and a feasible solution $\lambda_{c, t} \geq 0$ to the convex program \eqref{eq:ub_lp}.
\STATE For each $c \in [N]$, set $\eta_c = \sqrt{(\sum_{t \in C_c} \lambda_{c, t}^2) \log K}$.
\FOR{each round $t \in [T]$}
  \FOR{each order $C_c$ containing $t$}
    \STATE Let $t_1, t_2, \dots, t_w = t$ be the prefix of $C_c$ up to and including $t$. 
    \STATE Let $x_{t}^{(c)} \in \Delta_K$ be the strategy defined via (for any $i \in [K]$):

    \begin{equation}
        x_{t, i}^{(c)} = \frac{\exp\left(-\eta_c\sum_{u=1}^{w-1}\lambda_{c, t_{u}}\ell_{t_{u}, i}\right)}{\sum_{j=1}^{K}\exp\left(-\eta_c\sum_{u=1}^{w-1}\lambda_{c, t_{u}}\ell_{t_{u}, j}\right)}.
    \end{equation}
  \ENDFOR
  \STATE Play $x_t = \sum_{c ; t \in C_c} \lambda_{c, t}x_t^{(c)}$.
  \STATE Receive loss vector $\ell_t$ (and loss $\langle x_t, \ell_t\rangle$).
\ENDFOR
\end{algorithmic}
\end{algorithm}

Note that Algorithm \ref{alg:ub_alg} is a valid $\cS$-learning algorithm, since the action taken at time $t$ only depends on loss vectors $\ell_{t_{u}}$ for $t_{u} \in S_{t}$. Let $\UB(\cS)$ denote the optimal value of the upper bound program \eqref{eq:ub_lp}. The following theorem bounds the regret of Algorithm \ref{alg:ub_alg}.

\begin{theorem}\label{thm:ub_main}
If Algorithm \ref{alg:ub_alg} is run with an optimal solution $\blambda^*$ to the upper bound program \eqref{eq:ub_lp}, it incurs at most $O(\UB(\cS)\sqrt{\log K})$ regret.
\end{theorem}
\begin{proof}
Fix an action $x^* \in [K]$, and consider the regret of this algorithm (which we will call $\cA$) against action $x^*$ for some fixed loss sequence $\bell$. We can write this regret in the form:

\begin{eqnarray}
\Reg_{x^*}(\cA, \bell) &\coloneqq& \sum_{t=1}^{T}\langle x_{t} - x^*, \ell_t\rangle\notag\\
&=& \sum_{t=1}^{T} \sum_{c=1}^{N} \left\langle \lambda^*_{c, t}\left(x_{t}^{(c)} - x^*\right), \ell_t\right\rangle \label{eq:alg_proof1}\\
&=& \sum_{c=1}^{N} \left(\sum_{t=1}^{T}\left\langle x_{t}^{(c)} - x^*, \lambda^*_{c, t}\ell_t\right\rangle\right). \label{eq:alg_proof2}
\end{eqnarray}

Here in \eqref{eq:alg_proof1} we have used the fact that $\sum_{c=1}^{N}\lambda^*_{c, t} = 1$ for any fixed $t \in [T]$ (since the $\lambda^*_{c, t}$ satisfy the convex program \eqref{eq:ub_lp}). Note that, by definition, $x_{t}^{(c)}$ is the output of an instance $\cA_p$ of the Hedge algorithm initialized with learning rate $\eta_c$ on losses $\bell^{(c)}$ given by $\ell^{(c)}_t = \lambda^*_{c, t}\ell_t$; moreover, the summand of the RHS of \eqref{eq:alg_proof2} corresponding to a given $c$ is at most $\Reg(\cA_c, \bell^{(c)})$. Therefore, by Lemma \ref{lem:loss-dependent-regret}, we have that

$$\Reg_{x^*}(\cA, \bell) \leq \sum_{c=1}^{N} 2\sqrt{\left(\sum_{t=1}^T (\lambda^*_{c, t})^2\right)\log K} = 2\sqrt{\log K} \cdot \UB(\cS).$$

\noindent
Since $\Reg(\cA, \bell) = \max_{x^* \in \Delta_K} \Reg_{x^*}(\cA, \bell)$, the conclusion follows.
\end{proof}

Theorem \ref{thm:ub_main} leads to two natural questions:

\begin{enumerate}
    \item  First, although we can solve the convex program \eqref{eq:ub_lp} in time polynomial in $N$, $T$, and $K$, one might notice that the parameter $N$ is equal to the number of (maximal) orders in the temporal feedback graph $\cS$ and can be very large (possibly exponential in $T$). In fact, even specifying a solution to the convex program and running Algorithm \ref{alg:ub_alg} requires time polynomial in $N$ as written. Are there efficient learning algorithms (running in polynomial time in $T$ and $K$) for this problem?

    \item Secondly, does Algorithm \ref{alg:ub_alg} obtain the optimal regret bound for the $\cS$-learning problem?  That is, must any $\cS$-learning algorithm $\cA$ incur at least $\Omega(\UB(\cS)\sqrt{\log K})$ regret on some sequence of losses?
\end{enumerate}

The remainder of this paper will largely be concerned with providing answers to both of these questions (especially for the specific case of transitive feedback graphs $\cS$). For both of these questions, we will find it useful to examine the dual of the upper bound convex program, which we explore in the next section.

\subsection{The dual convex program and an efficient learning algorithm}

For multiple reasons, we will find it easier to work with the dual of the upper bound convex program. In contrast with the original upper bound program, which is a minimization problem involving (up to) $T\cdot N$ variables with $T$ constraints, the dual program is a maximization problem on $T$ variables with $N$ constraints. We state this program (which we call the \emph{upper bound dual program}) below. 

\begin{align}
\text{Maximize} \quad & \sum_{t=1}^{T} \mu_t \label{eq:ub_dual_lp} \\
\text{Subject to} \quad &\sum_{t \in C_c} \mu_t^2 \leq 1 \quad \text{ for all } c \in [N] \notag \\
& \mu_t \geq 0 \quad \text{ for all } t \in [T]\notag
\end{align}

The following theorem proves that the program \eqref{eq:ub_dual_lp} is in fact the dual of the upper bound program.

\begin{theorem}\label{thm:dual}
The optimal value of the upper bound dual program for temporal feedback graph $\cS$ is equal to $\UB(\cS)$.\footnote{For the sake of brevity, we defer many of the longer or more standard proofs to Appendix \ref{sec:omitted-main}.}
\end{theorem}

Having established duality, complementary slackness allows us to infer strong structural statements about the optimal primal solution given an optimal dual solution.

\begin{lemma}\label{lem:complementary-slackness}
Let $\bmu^*$ be an optimal solution to the upper bound dual program \eqref{eq:ub_dual_lp}. Then there exists a primal solution $\blambda^*$ to \eqref{eq:ub_lp} such that, for any $c \in [N]$ where

\begin{equation}\label{eq:complementary-slackness}
\sum_{t \in C_c}(\mu^*_{t})^2 < 1,
\end{equation}

\noindent
we have that $\lambda^*_{c, t} = 0$ for all $t \in [T]$, and for any $c \in [N]$ where

\begin{equation}\label{eq:complementary-slackness2}
\sum_{t \in C_c}(\mu^*_{t})^2 = 1,
\end{equation}

\noindent
we have that $\lambda^{*}_{c, t} = \rho_c\mu^*_t$, for some constant $\rho_c$. 
\end{lemma}

Given a solution $\bmu^*$ to the upper bound dual, let $\cC(\bmu^*) \subseteq [N]$ denote the set of orders $c$ where the equality \eqref{eq:complementary-slackness2} is tight. Lemma \ref{lem:complementary-slackness} implies that we only need to solve the upper bound program (and run Algorithm \ref{alg:ub_alg}) for orders $c \in \cC(\bmu^*)$. If $|\cC(\bmu^*)|$ is small, this can be far more efficient than naively running Algorithm \ref{alg:ub_alg}. In particular, since there are only $T$ variables in the dual program, we would naively expect only $T$ of the constraints to be binding, and therefore $|\cC(\bmu^*)|$ should generically equal $T$ (in which case we have an efficient solution).

Of course, for many temporal feedback graphs $\cS$, the dual program is not generic, and it can be the case that many (or all of) the constraints bind at optimality. The following lemma shows that, even in such cases, it is possible to find an optimal solution of the primal program that is supported on at most $T$ different orders.

\begin{lemma}\label{lem:basis}
Given any optimal solution $\bmu^*$ to the upper bound dual program, there exists a subset $\cB \subseteq \cC(\mu^*)$ with $|\cB| \leq T$ such that there exists a optimal solution $\blambda^*$ to the upper bound program with the property that $\lambda^*_{c, t} = 0$ if $C_c \not\in \cB$.
\end{lemma}
\begin{proof}
By Lemma \ref{lem:complementary-slackness}, we can restrict our attention to optimal solutions $\blambda^*$ where $\lambda^*_{c, t} = \rho_{c}\mu^*_t$ for $c \in \cC(\mu^*)$ (for some collection of values $\rho_{c} \geq 0$) and $\lambda^*_{c, t} = 0$ for $c \not\in \cC(\mu^*)$. Since $\sum_{t \in C_c} (\mu^*_t)^2 = 1$ for all $c \in \cC(\mu^*)$, we can rewrite the original upper bound primal program as the following linear program in the $\rho_c$:

\begin{align}
\text{Minimize} \quad & \sum_{c\in \cC(\mu^*)} \rho_c \label{eq:ub_lp_basis} \\
\text{Subject to} \quad &\sum_{ \substack{c \in \cC(\bmu^*) \\ \text{s.t. }t \in C_c}} \rho_c = 1/\mu_t \quad \text{ for all } t \in [T] \notag \\
& \rho_c \geq 0 \quad \text{ for all } c \in \cC(\bmu^*).\notag
\end{align}

But any extreme point of the linear program \eqref{eq:ub_lp_basis} will be the intersection of at least $|\cC(\bmu^*)|$ constraints, of which at most $T$ are not of the form $\rho_c = 0$. It follows that there is an optimal extreme point to this LP where at most $T$ of the $\rho_c$ are non-zero, and hence a $\blambda^*$ with the property we described.
\end{proof}

Given a feedback graph $\cS$, we call a subset $\cB$ of $\cC$ a \emph{basis} for $\cS$ if there exists an optimal solution to the upper bound program supported entirely on orders $c \in \cB$. Lemma \ref{lem:basis} shows that there always exists a basis of size at most $T$. Having an optimal solution with a small basis is valuable since it allows us to run Algorithm \ref{alg:ub_alg} more efficiently. 

\begin{lemma}\label{lem:sparse-runtime}
Let $\blambda$ be a feasible point for the upper bound program \eqref{eq:ub_lp} which is supported on a basis $\cB$ of size $|\cB| = B$. Then we can run Algorithm \ref{alg:ub_alg} on $\blambda$ in time $O(BK)$ per iteration ($O(BKT)$ overall).
\end{lemma}
\begin{proof}
We modify Algorithm \ref{alg:ub_alg} by maintaining one instance of Hedge for each of the $B$ non-zero orders (each of which takes $O(K)$ time to update per iteration). 
\end{proof}

Together with Lemma \ref{lem:basis}, this provides us with a partial answer to our first question.

\begin{corollary}\label{cor:efficient}
For any feedback graph $\cS$, there exists an $\cS$-learning algorithm which runs in time $O(KT)$ per iteration and incurs regret at most $O(\UB(\cS) \sqrt{\log K})$.
\end{corollary}

The catch, of course, is that actually coming up with the learning algorithm of Corollary \ref{cor:efficient} involves computing an optimal basis $\cB$ and associated primal solution $\blambda^*$, which may not be computationally efficient. In Section \ref{sec:transitive}, we will see that for the important class of transitive feedback graphs, it is possible to compute this sparse primal solution efficiently. 




\section{Lower bounds for online learning with temporal feedback graphs}\label{sec:lbs}

We now turn our attention to the second of the two questions: is our regret bound of $O(\UB(\cS)\sqrt{\log K})$ asymptotically tight? Throughout this section we will focus on the two-action setting ($K = 2$) for clarity of exposition. We expect that all lower bounds we present should extend to the $K$-action setting (with an additional factor of $\sqrt{\log K}$ in the lower bound). 

\subsection{A (naive yet efficient) lower bound program}

We begin by examining what happens when we try to extend the original lower bound proof for the standard problem of learning with experts to the temporal feedback graph setting. 

At a high-level, the original lower bound proof proceeds as follows. First, the adversary uniformly samples a bit $B \in \{-1, 1\}$ unknown to the learner. The adversary then fixes an $\eps \in (0, 1/2)$ and generates a sequence of i.i.d. random variables $X_1, X_2, \dots, X_T$ where each $X_t \in \{0, 1\}$ is drawn independently from a Bernoulli $(\frac{1}{2} + B\eps)$ distribution. The adversary then chooses losses defined via $\ell_t = (X_t, 1/2)$. 

Now, the learner cannot learn much about the random bit $B$ until they have seen at least $\Omega(1/\eps^2)$ samples from $\Bern(\frac{1}{2} + B\eps)$, and hence until they have seen at least $\Omega(1/\eps^2)$ loss vectors. But during that time, they incur $\Omega(\eps)$ regret per round. Altogether, this implies that the adversary can force the learner to incur a total regret of at least $\Omega(\eps \cdot \max(1/\eps^{2}, T))$. Picking $\eps = 1/\sqrt{T}$, we obtain the well-known $\Omega(\sqrt{T})$ lower bound.

We will start our discussion of lower bounds for the problem of $\cS$-learning with a similar approach. Again, the adversary will begin by uniformly sampling a bit $B \in \{-1, 1\}$. But now, to account for the additional potential asymmetry in the feedback structure across rounds, the adversary will select a sequence of ``scales'' $\eps_1, \eps_2, \dots, \eps_T \geq 0$, and sample the r.v. $X_t$ from the distribution $\Bern(\frac{1}{2} + B\gamma\eps_t)$ (for some ``global scale'' $\gamma > 0$). Finally, the adversary again sets $\ell_t = (X_t, 1/2)$.

Intuitively, we can understand the performance of this strategy as follows. As long as the learner cannot figure out the value of $B$ based on the losses observable at round $t$, the learner will incur an expected regret of $\Omega(\eps_t)$ in that round. In order to prevent the learner from figuring out the value of $B$ in round $t$, the amount of information leaked by the losses in $S_t$ about $B$ should be small. This quantity is roughly proportional to $\sum_{s \in S_t}\eps_s^2$ (in particular, observing $\ell_{s}$ leaks $\Theta(\eps_s^2)$ bits of information about $B$). This motivates writing down the following convex program \eqref{eq:lb_lp}, which we call the \emph{lower bound program}.

\begin{align}
\text{Maximize} \quad & \sum_{t=1}^{T} \eps_t \label{eq:lb_lp} \\
\text{Subject to} \quad &\sum_{t \in S_{t'}} \eps_t^2 \leq 1 \quad \text{ for all } t'\in [T] \notag \\
& \eps_t \geq 0 \quad \text{ for all } t \in [T]\notag
\end{align}

We will write $\LB(\cS)$ to denote the optimal value of the lower bound program. Note that this program shares many structural similarities with the upper bound dual program -- the main difference is that whereas in the upper bound dual program \eqref{eq:ub_dual_lp}, the constraints bound the sum of the squares of the variables over all \emph{orders} in $\cS$, here the sums are over all \emph{neighborhoods} $S_{t'}$ in $\cS$. Since every order $C_c$ is a subset of the neighborhood $S_{t'}$ of its last element, the constraints of the lower bound program are stronger than that of the upper bound dual (and so $\LB(\cS) \leq \UB(\cS)$, as we would expect).

The following theorem formalizes the above intuition and shows that $\LB(\cS)$ is indeed a valid lower bound for the $\cS$-learning problem (up to a universal constant factor).

\begin{theorem}\label{thm:lbone}
Every $\cS$-learning algorithm $\cA$ must incur worst-case regret $\Reg(\cA) \geq \LB(\cS)/100$.
\end{theorem}

One nice property of the lower bound program \eqref{eq:lb_lp} is that it is polynomial-sized, and can therefore be optimized efficiently. Another feature of this program is that, due to its similarity with the upper bound dual program, it immediately gives us the following multiplicative bound on its optimality.

\begin{theorem}\label{thm:path-cover}
Let $\cS$ be a temporal feedback graph where every neighborhood $S_t$ is contained in the union of at most $R$ orders $C_c$. Then $\UB(\cS)/\LB(\cS) \leq \sqrt{R}$ (and hence the regret bound of Algorithm \ref{alg:ub_alg} is within $\Omega(\sqrt{R})$ of optimal).
\end{theorem}
\begin{proof}
Let $\bmu^*$ be an optimal solution to the upper bound dual program. Note that in such a graph, $\mu_t$ would satisfy $\sum_{t\in S_{t'}} \mu_t^2 \leq R$, and therefore setting $\eps_{t} = \mu_t / \sqrt{R}$ forms a feasible solution to \eqref{eq:lb_lp} with value at least $\UB(\cS)/\sqrt{R}$ and at most $\LB(\cS)$. The conclusion follows.
\end{proof}

There are some classes of feedback graphs where the parameter $R$ in Theorem \ref{thm:path-cover} is $O(1)$ and hence where we can immediately conclude that Algorithm \ref{alg:ub_alg} is nearly optimal (one interesting such class of graphs is those arising from bounded recall, where each $S_t$ is itself an order). The downside, however, is that in general the gap between $\LB(\cS)$ and $\UB(\cS)$ can be quite large (at least $\Omega(T^{1/4})$ -- see Appendix \ref{sec:ub_lb_gap}). In the next section we present a stronger lower bound that comes from a much larger convex program.

\subsection{The independent set lower bound}\label{sec:independent}

One of the main reasons the lower bound in the previous section is lax is every loss which is visible to the learner at round $t$ reveals independent information about $B$. But since not all of these rounds can observe each other, this is inherently a little wasteful -- it would be better if we could ``reuse'' the same signal about $B$ across multiple different rounds.

To implement this idea, let $\Ii(\cS)$ denote the collection of independent sets of the (undirected version) of $\cS$. The adversary then begins by generating a random variable with some bias $Bw_I$ for every independent set $I \in \Ii(\cS)$. Finally, to generate the loss at round $t$, the adversary combines the random variables for sets $I$ containing $t$ (by e.g. taking their majority). This motivates us to write down the following convex program:

\begin{align}
\text{Maximize} \quad & \sum_{t=1}^T \sqrt{\sum_{\substack{I \in \Ii(\cS)\\\text{s.t. } t\in I}} w_{I}^2} \\
\text{Subject to} \quad &\sum_{\substack{I \in \Ii(\cS) \text{ s.t.}\\I \cap \cS_t \neq \emptyset}}w_{I}^2 \leq 1 \quad \text{ for all } t \in [T] \notag
\end{align}
    
We call the above program the \emph{independent set (lower bound) program}, and denote its optimal value by $\ILB(\cS)$. Note that if we restrict the set of feasible set to $w$ that take nonzero values on single nodes, then the independent set program reduces exactly to our original lower bound program, and thus implies that $\ILB(\cS) \geq \LB(\cS)$. But more importantly, this more general program still gives a lower bound on the achievable regret of any algorithm.
\begin{theorem}\label{thm:indepsets}
Every $\cS$-learning algorithm $\cA$ must incur worst-case regret $\Reg(\cA) \geq \ILB(\cS)/50$.
\end{theorem}

We leave it as an interesting open question to characterize the gap between $\ILB(\cS)$ and $\UB(\cS)$. We conjecture that this gap is at most a universal constant (thus showing that both Algorithm \ref{alg:ub_alg} and this lower bound are optimal up to constant factors).

\begin{conjecture}\label{conj:tight}
\sloppy{There exists a constant $\gamma$ such that for any temporal feedback graph $\cS$, $\UB(\cS)/\ILB(\cS) \leq \gamma$.}
\end{conjecture}

In the next section, we employ a variant of this technique to prove a tight lower bound for all transitive graphs, providing partial evidence for this conjecture.

\section{Transitive feedback graphs}
\label{sec:transitive}

In this section, we turn our attention to the case of transitive feedback graphs $\cS$. Recall that these are acyclic feedback graphs $\cS$ where if $u \in S_{t}$ and $v \in S_{u}$, then $v \in S_{t}$ (that is, if you can see loss $\ell_{u}$ at round $t$, you can also see all losses you could see at round $u$). 

\subsection{An efficient algorithm}

We will begin by showing that we can efficiently implement Algorithm \ref{alg:ub_alg} for any transitive feedback graph $\cS$. In particular, it suffices to show that we can efficiently find an optimal basis $\cB$ and associated primal solution $\blambda^*$ in time polynomial in $T$.

Note that for a transitive feedback graph $\cS$, any directed path $t_1 \rightarrow t_2 \rightarrow \dots \rightarrow t_w$ in $\cS$ forms an order. This allows us to construct an efficient separation oracle (and hence efficiently solve) the upper bound dual program \eqref{eq:ub_dual_lp}.

\begin{lemma}\label{lem:transitive-dual}
If $\cS$ is a transitive feedback graph, then we can find an optimal solution $\bmu^*$ to the upper bound dual program (to within $\eps$ additive error) in time $\poly(T, 1/\eps)$. Moreover, there exists an efficiently computable subgraph $\cS'$ of $\cS$ and two sets of rounds $\src(\cS')$ and $\dest(\cS')$ such that $\cC(\bmu^*)$ is equal to the set of all directed paths contained within $\cS'$ that start at a node in $\src(\cS')$ and end at a node in $\dest(\cS')$.
\end{lemma}
\begin{proof}
We will first provide an efficient separation oracle which, given a $\bmu \geq 0$, will either describe which of the order constraints $c \in [N]$ $\bmu$ violates, or report that $\bmu$ is a valid dual solution. With such an oracle, we can use the ellipsoid method to solve the convex program to within $\eps$ additive error in time $\poly(T, 1/\eps)$.

Since every directed path in $\cS$ corresponds to some order $C_c$, it suffices to be able to find the directed path $P$ in $\cS$ which maximizes $\sum_{t \in P} \mu_t^2$ (if this maximum is larger than $1$, we have a violating constraint given by the order $P$ corresponds to). But this is equivalent to computing the longest weighted path in a directed acyclic graph (where vertex $t$ has weight $\mu_t^2$), which can be solved efficiently in $O(T^2)$ time via dynamic programming.

Now, note that $C(\bmu^*)$ consists of all directed paths $P$ where $\sum_{t \in P} \mu_t^2 = 1$. We can use the same dynamic program to show that $C(\bmu^*)$ simply contains all paths in a subgraph $\cS'$ of $\cS$ that start in a source set $\src(\cS')$ and end in a target set $\dest(\cS')$. For each node $t \in [T]$, let $V_t$ be the maximum value of $\sum_{s \in P} \mu_{s}^2$ over any directed path $P$ ending at $t$. We can use the following procedure to construct $\cS'$:

\begin{itemize}
\item Start by adding the set of vertices $t$ where $V_{t} = 1$ to $\cS'$.
\item For every round $t$ in $\cS'$ that hasn't been processed, find all ancestors $s \in S_t$ with the property that $V_{t} - V_{s} = \mu_{t}^2$. For each such ancestor $s$, add the edge $s \rightarrow t$ to $\cS'$ (i.e., add $s$ to $S'_t$), and process $s$ if it has not already been processed.
\item Finally, let $\src(\cS')$ equal the set of rounds $t \in \cS'$ where $V_{t} = \mu_{t}^2$, and let $\dest(\cS')$ equal the set of rounds $t$ in $\cS'$ where $V_{t} = 1$ (equivalently, these are the sources and sinks of the DAG $\cS'$).
\end{itemize}

To see why this procedure works, note that any path $P = (t_1, t_2, \dots, t_k)$ from $\src(\cS')$ to $\dest(\cS')$ along edges of $\cS'$ will satisfy $\sum_{t \in P} \mu_{t}^2 = \sum_{i=1}^{k} (V_i - V_{i-1}) = V_k = 1$. Moreover, any edge this algorithm does not select cannot possibly be included in a path in $C(\bmu^*)$ (the sum of $\mu_t^2$ over a path from $t_i$ to $t_j$ containing this edge must be strictly less than $V_{j} - V_{i} \leq 1$). It follows that $C(\bmu^*)$ simply contains all source-destination paths in $\cS'$.
\end{proof}

We can now use the dual solution (and the characterization of $\cC(\bmu^*)$) provided by Lemma \ref{lem:transitive-dual} to find an efficient, sparse solution to the primal.

\begin{theorem}\label{thm:transitive-main}
If $\cS$ is a transitive feedback graph, then we can find an optimal solution $\blambda^*$ to the upper bound program (to within additive $\eps$ error) supported on a basis $\cB$ of size $|\cB| \leq T$ in time $\poly(T, 1/\eps)$.
\end{theorem}
\begin{proof}
By Lemma \ref{lem:transitive-dual}, we can efficiently construct a dual solution $\bmu^*$ and corresponding set $\cC(\bmu^*)$. Our approach will be to use this to find a sparse solution to the linear program \eqref{eq:ub_lp_basis} in the proof of Lemma \ref{lem:basis} in the $\rho_c$ random variables (recall that such a solution characterizes an optimal primal solution via $\lambda^{*}_{c, t} = \rho_c\mu^*_t$ for each $c \in \cC(\bmu^*)$, $t \in C_c$). 

The linear program in \eqref{eq:ub_lp_basis} has $|\cC(\bmu^*)|$ random variables, and therefore it would be inefficient to solve directly. Instead, we will show that we can use the structure of $\cC(\bmu^*)$ (as the set of all source-destination paths in the subgraph $\cS'$) to rewrite it as a flow problem. Indeed, consider the following linear program in the variables $f_{e}$ (for each edge $e = (s, t)$ in the edge set $E(\cS')$ of $\cS'$):

\begin{align}
\text{Minimize} \quad & \sum_{\substack{t \in \dest(\cS')\\ (s,t) \in E(\cS')}} f_{s, t} \label{eq:ub_lp_flow} \\
\text{Subject to} \quad &\sum_{s \mid (s, t) \in \cS'} f_{s, t} = 1/\mu_t \quad \text{ for all } t \in [T] \notag \\
&\sum_{s \mid (s, t) \in \cS'} f_{s, t} = \sum_{s \mid (t, s') \in \cS'} f_{t, s'} \quad \text{ for all } t \in \cS' \setminus (\src(\cS') \cup \dest(\cS')) \notag \\
& f_{s, t} \geq 0 \quad \text{ for all } (s, t) \in E(\cS').\notag
\end{align}

The linear program \eqref{eq:ub_lp_flow} is a polynomial-sized LP, so we can solve it efficiently. At the same time, it is equivalent to the linear program \eqref{eq:ub_lp_basis} in the following sense: first, given any solution $\rho_{c}$ to \eqref{eq:ub_lp_basis}, if we let $f_{e} = \sum_{c \mid e \in c} \rho_{c}$ equal the sum of $\rho_{c}$ over all paths $C_c$ that contain $e$, then $f_e$ satisfies \eqref{eq:ub_lp_flow} (with the same objective value). Conversely, given any flow $f_e$ solving \eqref{eq:ub_lp_flow}, we can decompose it into a positive combination of source-destination paths. If we let $\rho_{c}$ be the weight of the path $C_c$ in this decomposition, we can likewise check that $\rho_{c}$ satisfies \eqref{eq:ub_lp_basis} (also with the same objective value).

There are well-known efficient procedures for flow-decomposition (see e.g. \cite{ahuja1988network}), which take a flow $f_{e}$ and return a positive combination of at most $O(|E(\cS')|)$ paths. By doing this, we can obtain an optimal solution $\rho_{c}$ to \eqref{eq:ub_lp_basis} supported on a basis of size at most $O(T)$. We can then shrink this to an optimal basis $\cB$ of size at most $T$ by removing all other variables from the LP \eqref{eq:ub_lp_basis} and using any method for finding a basic feasible solution to the resulting program (e.g. optimizing a random linear functional over the face containing the optimum).
\end{proof}

Theorem \ref{thm:transitive-main} implies that given any transitive feedback graph $\cS$, we can efficiently construct the algorithm of Corollary \ref{cor:efficient} that runs in time $O(KT)$ per iteration and incurs regret at most $O(\UB(\cS)\sqrt{\log K})$. 

\subsection{An asymptotically tight lower bound}

Finally, we prove a lower bound of $\Omega(\UB(\cS))$ on the regret of any algorithm when the feedback graph $\cS$ is transitive. This combined with our upper bound in Corollary~\ref{cor:efficient} settles the optimal learning rate for the case of transitive graphs. 

\begin{theorem}[Tight lower bound for transitive $\cS$]\label{thm:bmlb}
For any transitive graph $\cS$, every $\cS$-learning algorithm $\cA$ must incur worst-case regret $\Reg(\cA) \geq \UB(\cS)/50$.
\end{theorem}

\begin{proof}
To show this lower bound, we use the fact that the value of the upper bound program $\UB(\cS)$ is equal to its dual~\eqref{eq:ub_dual_lp}. Then, for every instance $(\mu_t)_{t=1}^T$ of~\eqref{eq:ub_dual_lp}, we show $\Omega(\sum_{t=1}^T \mu_t)$ as a lower bound for the value of the regret. To show this lower bound, we take a similar strategy as in the proof of Theorem~\ref{thm:lbone}; the adversary flips a coin $B\in \{-1,+1\}$ and then consider loss vectors $\ell_t = (X_t, \frac{1}{2})$ where $X_t$ is a $\gamma\epsilon_t$-biased Bernoulli variable, where the bias is to one if $B=1$ and to zero if $B=-1$. Here again the adversary attempts to use shared randomness between $\ell_t$'s to block the chances of the player to learn $B$. At the same time, the adversary has to be careful not to reveal any information about the $\ell_t$ given the loss vectors $\ell_s$'s, $s\in \cS_t$ which it can observe at time $t$. Here, the adversary exploits the transitive nature of $\cS$ to cook up these random variables and considers a linearly shifted Brownian motion with rate $\gamma$:
\begin{align*}
    L_t = \gamma B t + B_t.
\end{align*}
The adversary defines each $X_t$ as a positivity indicator variable of a chunk of the process $L$. Namely, for times $p_t \geq q_t$,
\begin{align*}
    X_t = \mathrm 1\{L_{p_t} - L_{q_t} \geq 0\}
\end{align*}

Specifically, for each time $t \in [T]$, the adversary defines 
\begin{align*}
    &q_t = \max_{s\in \cS_t}\{p_s\},\\
    &p_t = q_t + \mu_t^2.
\end{align*}
With this definition, first we show by induction that $p_t \leq 1$ for all $t$. We strengthen the hypothesis of induction and concurrently show the argument that for any $t$, there is an ordered clique $p$ in $\cS$ ending at $t$ and with $\sum_{s\in p}\mu_s^2 = p_t$. The hypothesis of induction is trivial since $q_1 = 0$ and $p_1 = \mu_1^2$. Now for arbitrary $t\leq T$, let $s\in \cS_t$ be the index with maximum $p_s$ in $\cS_t$. From the hypothesis of induction we know there is an ordered clique $p$ ending at $s$ with $\sum_{s'\in p} \mu_{s'}^2 = p_s$. Now from definition we have $p_t = p_s + \mu_t^2$. Therefore, $p_t = \sum_{s\in p'}\mu_{s}^2$ where $p'$ is the ordered clique of $p$ concatenated with $t$. This shows the step of induction. Finally, note that from the argument that we showed, that there is an ordered clique $p$ ending at $t$ with $\sum_{s \in p}\mu_s^2 = p_t$, we conclude $p_t = \sum_{s \in p}\mu_s^2 \leq 1$ because of the constraint in the dual of the upper bound program. 

The second observation is that with this definition, $X_t$ becomes independent of $X_s$ for $s\in \cS_t$. This follows from the independence of disjoint increments of Brownian motion. Next, we show that with small enough choice of constant $\gamma$, the adversary cannot distinguish between $B=\pm 1$ cases with constant probability. For this, similar to the proof of Theorem~\ref{thm:lbone} it is enough to bound the total variation distance between $ Q^{-1}(\cup_{s \in \cS_t} (L_{p_s}-L_{q_s}))$ and $Q^{+1}(\cup_{s \in \cS_t} (L_{p_s}-L_{q_s}))$, where here we use the notation $Q^{+}(X)$ and $Q^{-}(X)$ to refer to the distribution of the random variable (or more generally random process) $X$ given $B=1$ and $B=-1$, respectively. But again from the data processing inequality, we have
\begin{align}
  &\text{TV}\Big(Q^{-}(\cup_{s \in \cS_t} (L_{p_s}-L_{q_s})) , Q^{+}(\cup_{s \in \cS_t} (L_{p_s}-L_{q_s}))\Big)\\ 
  &\leq \sqrt{D\Big(Q^{-}(\cup_{s \in \cS_t} (L_{p_s}-L_{q_s})) || Q^{+}(\cup_{s \in \cS_t} (L_{p_s}-L_{q_s}))\Big)}\\
  &\leq \sqrt{D\Big(Q^{-}\big(L[0:1]\big)|| Q^{+}\big(L[0:1]\big)\Big)},\label{eq:bounddd}
\end{align}
where in the last Equation, $Q^{\pm}(L[0:1])$ refers to the distribution corresponding to the whole sample path of the process $L$ in the interval $[0,1]$ and we used the fact that for all times $s$, $0 \leq q_s\leq p_s \leq 1$.

Next, we use the following Lemma, proved in Appendix~\ref{sec:gkl}, to upper bound the RHS in Equation~\eqref{eq:bounddd}. This Lemma provides a formula for the KL divergence of two shifted Brownian motions.
\begin{lemma}[KL divergence between shifted Brownian motions]\label{lem:brownianKL}
    For linearly shifted Brownian motions $L_t = \gamma t + B_t$, the KL divergence between the measures corresponding to $L_t$ and $B_t$ in the interval $[0,1]$ is equal to
    \begin{align*}
        D\Big(Q\big(L[0:1]\big) || Q\big(B[0:1]\big)\Big) = \gamma^2/2.
    \end{align*}
\end{lemma}

Applying Lemma~\ref{lem:brownianKL}, we get
\begin{align*}
    D\Big(Q^{-}\big(L[0:1]\big)|| Q^{+}\big(L[0:1]\big)\Big) \leq 2\gamma^2,
\end{align*}
Plugging this into Equation~\eqref{eq:bounddd}:
\begin{align*}
    \text{TV}\Big(Q^{-1}(\cup_{s \in \cS_t} (L_{p_s}-L_{q_s})) , Q^{+1}(\cup_{s \in \cS_t} (L_{p_s}-L_{q_s}))\Big)  \leq 
    2\gamma.
\end{align*}
Therefore, similar to the proof of Theorem~\ref{thm:indepsets},
\begin{align*}
    \E[\epsilon_t|x_t - x^*_t|] \geq \gamma \epsilon_t(1/2 - \gamma) = \frac{\gamma\epsilon_t}{4}.
\end{align*}
by picking $\gamma = 1/4$. Moreover, from Lemma~\ref{lem:epstlowerbound} we have $\epsilon_t \geq \frac{\mu_t}{\sqrt{2\pi}}$. Therefore, 
\begin{align*}
    \Reg(\cA) \geq \gamma \sum_{t=1}^T\frac{\mu_t}{4\sqrt{2\pi}} = \frac{\sum_t \mu_t}{16\sqrt{2\pi}} = \frac{\UB(\cS)}{16\sqrt{2\pi}}.
\end{align*}
\end{proof}

We briefly remark on the connection between Theorem \ref{thm:bmlb} and the independent set program of Section \ref{sec:independent}. Although we have presented Theorem \ref{thm:bmlb} in a completely self-contained way, we can view the construction in the proof of this theorem as constructing a feasible point to the independent set program.

Indeed, in the proof of the above theorem, we associate to each round $t \in [T]$ an interval $[p_t, q_t]$ contained within the unit interval. The intersection of all these intervals induces a partition of the unit interval into sub-intervals, each of which is labeled with a subset of rounds of $[T]$. In addition, by our construction of these intervals, two intervals $[p_s, q_s]$ and $[p_t, q_t]$ can intersect only if $s$ and $t$ are not adjacent in $\cS$. Therefore, each sub-interval is actually labeled with some \emph{independent set} $I$ belonging to $\Ii(\cS)$. Taking the weight $w_I$ to be the square root of the length of this interval, the analysis in the above proof implies that $w_I$ form a feasible solution to the independent set program with value equal to $\sum_{t} \mu_t = \UB(\cS)$.

\bibliographystyle{abbrvnat}
\bibliography{main}

\appendix

\section{The lower bound $\LB(\cS)$ is not tight}
\label{sec:ub_lb_gap}

In this appendix, we show that the gap between the value of the upper bound program $\UB(\cS)$ and the lower bound program $\LB(\cS)$ can grow without bound.

\sloppy{
\begin{theorem}\label{thm:ub_lb_gap}
For any $T$, there exists a temporal feedback graph $\cS$ on $T$ rounds where $\UB(\cS)/\LB(\cS) \geq \Omega(T^{1/4})$.
\end{theorem}
\begin{proof}
Consider the feedback graph $\cS$ formed by batched learning setting where the time horizon $T$ is divided into $\sqrt{T}$ batches of $\sqrt{T}$ rounds each (so, $S_{t}$ only contains rounds $s < \lfloor t / \sqrt{T}\rfloor \sqrt{T}$).

Since each order in $\cS$ contains at most $\sqrt{T}$ rounds, one feasible solution for the upper bound dual program is to set $\mu_{t} = T^{-1/4}$ for all $t \in [T]$, which implies $\UB(\cS) \geq T^{3/4}$.

On the other hand, note that in the lower bound program, we have that $\sum_{t \in S_T}\eps_t^2 \leq 1$. Since $S_t$ contains at most $T$ elements, this implies that $\sum_{t \in S_T} \eps_t \leq \sqrt{T}$. But on the other hand, $S_T$ can see all the rounds except the rounds in the very last batch (of which there are at most $\sqrt{T}$). So $\sum_{t \not\in S_t} \eps_t \leq \sqrt{T}$, and therefore $\LB(\cS) \leq \sum_{t} \eps_t \leq 2\sqrt{T}$. 

It follows that $\UB(\cS)/\LB(\cS) \geq T^{1/4}/2 = \Omega(T^{1/4})$, as desired.
\end{proof}
}

\section{Lemmas from probability and information theory}
\label{sec:omitted}

In this appendix, we establish some standard results from probability and information theory that we make use of in our proofs of our lower bounds.

\subsection{Bound of KL-divergence for Bernoulli random variables}

\begin{lemma}\label{lem:kl-computation}
Fix a $\delta$ such that $0 \leq \delta \leq 1/4$. If $Q^{+} = \Bern(1/2 + \delta)$, $Q^{-} = \Bern(1/2 - \delta)$, then $D(Q^{-} \parallel Q^{+}) \leq 12\delta^2$.
\end{lemma}
\begin{proof}
Computing the KL divergence explicitly, we have that:

\begin{eqnarray*}
D(Q^{+} \parallel Q^{-}) &=& \left(\frac{1}{2} - \delta\right)\log\frac{\frac{1}{2} - \delta}{\frac{1}{2} + \delta} + \left(\frac{1}{2} + \delta\right)\log\frac{\frac{1}{2} + \delta}{\frac{1}{2} - \delta}\\
&=& 2\delta(\log(1 + 2\delta) - \log(1 - 2\delta))\\
&\leq & 2\delta \cdot (6\delta) = 12\delta^2.
\end{eqnarray*}

Here we have used the fact that $\log(1 + x) - \log(1 - x) \leq 3x$ for $x \in [0, 1/2]$. 
\end{proof}

\subsection{From Gaussian to biased Bernoulli}
\begin{lemma}\label{lem:epstlowerbound}
        Given a Gaussian variable $Y \sim \mathcal N(c,1)$ with $c > 0$, and Bernoulli variable $X = \mathrm 1(Z\geq 0)$, we have 
        \begin{align*}
            \mathbb P(X = 1) \geq 1/2 + c/(2\pi).
        \end{align*}
    \end{lemma}
\begin{proof}
    Note that
    \begin{align*}
        \mathbb P(Z \geq 0) &= \frac{1}{2} + \int_{0}^c \frac{1}{\sqrt{2\pi}}e^{-(z-c)^2/2}\\
        &\geq \frac{1}{2} + \int_{0}^c \frac{1}{\sqrt{2\pi}}e^{-c^2/2}\\
        &\geq \frac{1}{2} + \frac{c}{\sqrt{2\pi}}(1-c^2/2)\\
        &\geq \frac{1}{2} + \frac{c}{2\pi}.
    \end{align*}
    This completes the proof.
\end{proof}

\subsection{KL divergence in Gaussian processes}\label{sec:gkl}
\begin{lemma}[KL divergence between two Gaussians]
    \begin{align*}
        D(\mathcal N(\mu_1, 1) || \mathcal N(\mu_2, 1))
        = (\mu_1 - \mu_2)^2/2.
    \end{align*}
\end{lemma}
    \begin{proof}
        We can write
        \begin{align*}
            D(\mathcal N(\mu_1, 1) || \mathcal N(\mu_2, 1))
            &= \mathbb E_{N(\mu_1, 1)} \ln(\frac{e^{-(y-\mu_1)^2/2}}{e^{-(y-\mu_2)^2/2}})\\
            &= \mathbb E_{N(\mu_1, 1)} ((y-\mu_2)^2/2 - (y-\mu_1)^2/2)\\
            &= (\mu_1 - \mu_2)^2/2.
        \end{align*}
    \end{proof}

\begin{lemma}[Restatement of Lemma~\ref{lem:brownianKL}]
    For a linearly shifted Brownian motions $L_t = \gamma t + B_t$, the KL divergence between the measures corresponding to $L_t$ and $B_t$ in the interval $[0,1]$ is equal to
    \begin{align*}
        D\Big(Q\big(L[0:1]\big) || Q\big(B[0:1]\big)\Big) = \gamma^2/2.
    \end{align*}
\end{lemma}
\begin{proof}
    From the Girsanov theorem applied to the exponential martingale of the process $\gamma B_t$, we have
    \begin{align*}
         \frac{dQ\big(B[0:1]\big)}{dQ\big(L[0:1]\big)} = e^{\gamma\int_0^1 dB_t - \frac{1}{2}\int_0^1 \gamma^2 dt}, 
    \end{align*}
   which implies
   \begin{align*}
       D\Big(Q\big(L[0:1]\big) || Q\big(B[0:1]\big)\Big) = \mathbb E_{B[0:1]} \Big(\gamma\int_0^1 dB_t - \frac{1}{2}\int_0^1 \gamma^2 dt\Big) = \frac{\gamma^2}{2}.
   \end{align*}
\end{proof}

\section{Omitted proofs}
\label{sec:omitted-main}

\subsection{Proof of Theorem~\ref{thm:dual}}
\begin{proof}
We will show that the upper bound dual program arises from Lagrangifying the original upper bound program (and thus this equality follows as a consequence of strong duality).

We will begin by weakening the upper bound program \eqref{eq:ub_lp} by replacing the strict equality $\sum_{c=1}^{N} \lambda_{c, t} = 1$ with the weak inequality $\sum_{c=1}^{N} \lambda_{c, t} \geq 1$. Note that this does not change the optimal value of the LP (as if this inequality is strict for a specific $t$, one can always improve the objective by decreasing one of the non-zero $\lambda_{c, t}$ while not altering any of the other constraints). By Lagrangifying these constraints, we have that

\begin{equation}\label{eq:duality1}
\UB(\cS) = \min_{\blambda \geq 0} \max_{\bmu \geq 0} \left[\sum_{c=1}^{N} \sqrt{\sum_{t \in C_c} \lambda_{c, t}^2} + \sum_{t=1}^{T} \mu_t \left(1 - \sum_{c=1}^{N} \lambda_{c, t}\right)\right].    
\end{equation}

Now, note that in this convex program the objective is convex, all the constraints are affine, and the program is always feasible (for every round $t$ there is at least one order containing it, namely the singleton order $\{t\}$). Therefore we can apply the theorem of strong duality (see Chapter 28 of \cite{rockafellar1970convex}), and interchange the order of minimum and maximum in \eqref{eq:duality1}.

\begin{equation}\label{eq:duality2}
\UB(\cS) = \max_{\bmu \geq 0}\min_{\blambda \geq 0} \left[\sum_{c=1}^{N} \sqrt{\sum_{t \in C_c} \lambda_{c, t}^2} + \sum_{t=1}^{T} \mu_t \left(1 - \sum_{c=1}^{N} \lambda_{c, t}\right)\right].    
\end{equation}

\noindent
We can in turn rewrite \eqref{eq:duality2} in the following form:

\begin{equation}\label{eq:duality3}
\UB(\cS) = \max_{\bmu \geq 0}\left[\left(\sum_{t=1}^{T}\mu_t\right) + \min_{\blambda \geq 0} \sum_{c=1}^{N} \left(\sqrt{\sum_{t \in C_c} \lambda_{c, t}^2} - \sum_{t \in C_c} \lambda_{c, t}\mu_t\right)\right].    
\end{equation}

\noindent
Note that the terms in the internal sum pertaining to $C_c$ only depend on the variables $\lambda_{c, t}$. We can therefore interchange the order of this sum and min in \eqref{eq:duality3} and obtain:

\begin{equation}\label{eq:duality4}
\UB(\cS) = \max_{\bmu \geq 0}\left[\left(\sum_{t=1}^{T}\mu_t\right) + \sum_{c=1}^{N} \min_{\blambda_c \geq 0} \left(\sqrt{\sum_{t \in C_c} \lambda_{c, t}^2} - \sum_{t \in C_c} \lambda_{c, t}\mu_t\right)\right]
\end{equation}

\noindent
(where $\blambda_p$ represents the $|C_c|$ variables of the form $\lambda_{c, t}$). 

Let us now consider each of these inner optimization problems of the form 

$$\min_{\blambda_c \geq 0} \left(\sqrt{\sum_{t \in C_c} \lambda_{c, t}^2} - \sum_{t \in C_c} \lambda_{c, t}\mu_t\right).$$

Note that if we restrict the domain of $\blambda_c$ to the subdomain where $\sum_{t \in C_c} \lambda_{c, t}^2 = R^2$, this expression is minimized when each $\lambda_{c, t}$ is proportional to $\mu_t$. Specifically, it is minimized when

$$\lambda_{c, t} = \frac{\mu_{t}}{\sqrt{\sum_{t' \in C_c} \mu_{t'}^2}}\cdot R,$$

\noindent
at which point the expression has value equal to

$$R\left(1 - \sqrt{\sum_{t \in C_c}\mu_t^2}\right).$$

\noindent
We therefore have two cases depending on our choice of $\bmu$:

\begin{itemize}
    \item If $\sum_{t \in C_c}\mu_t^2 > 1$, then the value of this inner minimization problem is $-\infty$ (we can set $R$ arbitrarily large). 
    \item Otherwise, the value of this inner minimization problem is $0$, attained when $R = 0$.
\end{itemize}

It follows that any feasible $\bmu$ (that causes the outer-maximization problem to have finite value) must satisfy $\sum_{t \in C_c} \mu_t^2 \leq 1$ for each $c \in [N]$. If $\bmu$ is feasible, then the value of the inner expression is just $\sum_{t=1}^{T} \mu_t$. But note that this is precisely a description of the upper bound dual program \eqref{eq:ub_dual_lp}. The result follows.
\end{proof}

\subsection{Proof of Lemma \ref{lem:complementary-slackness}}
\begin{proof}
As in the proof of Theorem \ref{thm:dual}, we write the Lagrangified objective:

$$\cL(\blambda, \bmu) = \sum_{c=1}^{N} \sqrt{\sum_{t \in C_c} \lambda_{c, t}^2} + \sum_{t=1}^{T} \mu_t \left(1 - \sum_{c=1}^{N} \lambda_{c, t}\right) = \sum_{t=1}^{T} \mu_t + \sum_{c=1}^{N}\left(\sqrt{\sum_{t\in C_c} \lambda_{c, t}^2} - \sum_{t \in C_c}\mu_{t}\lambda_{c, t}\right).$$

From the properties of strong duality, we know that any optimal primal solution $\blambda^*$ must satisfy $\cL(\blambda^*, \bmu^*) \leq \cL(\blambda, \bmu^*)$ for any other $\blambda \geq 0$ (not necessarily feasible). Assume to the contrary that the inequality \eqref{eq:complementary-slackness} holds but we have that $\lambda^*_{c, t} > 0$ for some $t$. Then by Cauchy-Schwartz, we have that

$$\sum_{t \in C_c}\mu_{t}\lambda_{c, t} \leq \sqrt{\sum_{t\in C_c} \mu_{t}^2} \cdot \sqrt{\sum_{t\in C_c} \lambda_{c,t}^2} < \sqrt{\sum_{t\in C_c} \lambda_{c,t}^2}.$$

It follows that if we construct $\blambda'$ by taking $\blambda^*$ and setting $\lambda'_{c, t} = 0$ for all $t \in [T]$, then $\cL(\blambda^*, \bmu^*) > \cL(\blambda', \bmu^*)$, contradicting our assumption. Similarly, if the equality \eqref{eq:complementary-slackness2} holds but $\lambda^{*}_{c, t}$ is not proportional to $\mu_t$, then if we set $\lambda'_{c, t} = \mu_t \cdot \sqrt{\left(\sum_{t' \in C_c} \lambda_{t', c}^2\right) / \left(\sum_{t' \in C_c} \mu_{t'}^2\right)}$, we again find that $\cL(\blambda^*, \mu^*) > \cL(\blambda', \mu^*)$, contradicting our assumption.
\end{proof}

\subsection{Proof of Theorem \ref{thm:lbone}}
\begin{proof}
Let $\eps_t$ be an optimal solution to the lower bound program \eqref{eq:lb_lp}. Set $\gamma = 1/10$, and consider the distribution over loss vectors $\bell$ induced by the process described above (where the adversary first selects $B$ uniformly from $\{-1, 1\}$, selects each $X_t$ independently from $\Bern(1/2 + B\gamma\eps_t)$, and sets $\ell_t = (X_t, 1/2)$). We will show that when faced with a loss vector $\bell$ sampled from this process, any $\cS$-learning algorithm $\cA$ must incur $\Omega(\LB(\cS))$ expected regret (from which it follows that there exists a specific loss vector $\bell$ on which it incurs this much regret).

To prove this, note that when $\cA$ is deciding the action $x_t$ to take at time $t$, it can see the variables $X_{s}$ for $s \in S_t$ -- we will denote this set of random variables as $X_{S_t}$. If $B = -1$, these r.v.s are distributed according to a distribution $Q^{-}_{S_t}$, and if $B = 1$, these r.v.s are distribution according to a different distribution $Q^{+}_{S_t}$. We will upper bound the KL divergence $D(Q^{-}_{S_t} \parallel Q^{+}_{S_t})$; this in turn will allow us to apply Pinsker's inequality to upper bound the probability $\cA$ can successfully distinguish between the cases $B = -1$ and $B = 1$, from which we can lower bound the regret incurred by $\cA$.

Since $Q^{-}_{S_t}$ is a product distribution over $|S_t|$ independent Bernoulli r.v.s (and likewise for $Q^{+}_{S_t}$), by the chain rule for KL divergence we have that

$$D(Q^{-}_{S_t} \parallel Q^{+}_{S_t}) = \sum_{s \in S_t} D(Q^{-}_{s} \parallel Q^{+}_s),$$

\noindent
where $Q^{-}_{s}$ is simply the Bernoulli distribution $\Bern(1/2 - \gamma\eps_s)$ and $Q^{+}_s$ is the oppositely biased Bernoulli distribution $\Bern(1/2 + \gamma\eps_s)$. We can compute that for these distributions, $D(Q^{-}_{s} \parallel Q^{+}_s) \leq 12\gamma^2\eps_s^2$ (see Lemma \ref{lem:kl-computation} in the Appendix), so it follows that $D(Q^{-}_{S_t} \parallel Q^{+}_{S_t}) \leq \sum_{s \in S_t} 12\gamma^2\eps_s^2$, which in turn is at most $0.12$ (since $\gamma = 0.1$ and $\eps_s$ satisfy \eqref{eq:lb_lp}).

By Pinsker's inequality, the total variation distance between $Q^{-}_{S_t}$ and $Q^{+}_{S_t}$ is therefore at most 

$$\sqrt{(1/2)D(Q^{-}_{S_t} \parallel Q^{+}_{S_t})} \leq \sqrt{0.06} \leq 0.3.$$ 

Now, consider the probability $x_t \in [0, 1]$ that $\cA$ plays the first action at round $t$. Let $x^*_t = 0$ if $B = 1$ and $x^*_t = 1$ if $B = -1$. Then the expected regret $\cA$ incurs in this round against this adversary is $(X_t - 1/2)(x_t - x^*_t)$. In expectation, this is at least $\gamma\eps_{t} \E[|x_t - x^*_t|]$. But (applying Le Cam's method), $\E[|x_t - x^*_t|] = (1/2)\E[x_t \mid B = 1] + (1/2)\E[1 - x_t \mid B = -1] = (1/2) + (1/2)(\E[x_t \mid B = 1] - \E[x_t \mid B = -1])$. Since $x_t$ depends solely on $X_{S_t}$ (which are drawn from $Q^{-}_{S_t}$ when $B = 1$ and from $Q^{+}_{S_t}$ when $B = -1$), this second term is at most the total variation distance in magnitude, and thus $\E[|x_t - x^*_t|] \geq 0.5 - 0.5 \cdot 0.3 \geq 0.3$. 

It follows that the expected regret $\cA$ incurs this round is at least $0.3\gamma\eps_t \geq \eps_t/100$. Over all rounds, the expected regret $\cA$ incurs is therefore at least $\sum_{t}\eps_t/100 = \LB(\cS) / 100$, as desired.
\end{proof}

\subsection{Proof of Theorem~\ref{thm:indepsets}}
\begin{proof}
    The basic idea is similar to the proof of Theorem~\ref{thm:lbone}; the adversary again begins by sampling uniformly a random bit $B\in \{-1,+1\}$. For each independent set $I\in \Ii(\cS)$, the adversary samples a Gaussian random variable $Y_{I} \sim \mathcal N(\gamma Bw_{I}^2, \gamma w_{I}^2)$ (for some fixed $\gamma > 0$ to be decided later). Then, for each time $t$, she defines the Gaussian random variable $Z_t$ as the sum of Gaussian random variables $Y_I$ of the independent sets $I$ that include $t$ (letting $\Ii_t(\cS)$ denoting this sub-collection of independent sets):
    \begin{align*}
        Z_t = \sum_{I \in \Ii_t(\cS)} Y_I.
    \end{align*}
    Note that the mean and variance of $Z_t$ is equal to the sum of the means and variances of the $Y_{I}$ r.v.s, which are $\gamma B\sum_{I \in \Ii_t(\cS)}w_{I}^2$ and $\sum_{I \in \Ii_t(\cS)}w_{I}^2$, respectively. Finally, the adversary defines the loss at time $t$ as $\ell_t = (X_t, \frac{1}{2})$, where $X_t = \mathrm{1}(Z_t \geq 0)$ is a Bernoulli random variable with some bias $\gamma\epsilon_t$.

    By Lemma \ref{lem:epstlowerbound} (proved in Appendix~\ref{sec:omitted}), we can bound $\epsilon_t \geq \frac{1}{\sqrt{2\pi}} \cdot \sqrt{\sum_{I \in \Ii_t(\cS)}w_{I}^2}$ from our mean / variance calculations for $Y_i$.
    
    Let $Q^-_t$ and $Q^+_t$ be the distribution of $X_t$ given $B = -1$ and $B = +1$, respectively. For a subset of times $S \subseteq [T]$, let $Q^-_S$ and $Q^+_S$ be the product of distributions $Q^-_t$ and $Q^+_t$ for all $t \in S$, respectively.  
    The first component of the proof, similar to the proof of Theorem~\ref{thm:lbone}, is to bound the KL divergence $D(Q^-_{\cS_t} \parallel Q^+_{\cS_t})$. But by the data processing inequality, we can relate it to the distribution of the $Y_{I}$:
    \begin{align*}
        D(Q^-_{\cS_t} \parallel Q^+_{\cS_t}) \leq 
        D\left(\bigotimes_{\cS_t \cap I \neq \emptyset} Q^{-1}(I) \parallel \bigotimes_{\cS_t \cap I \neq \emptyset} Q^{+1}(I)\right),
    \end{align*}
    where $Q^{-1}(I)$ and $Q^{+1}(I)$ refer to the distribution of $Y_I$ given $B=-1$ and $B=1$, respectively.
    Now from the independence of the $Y_I$,
    \begin{align*}
        D\left(\bigotimes_{\cS_t \cap I \neq \emptyset} Q^{-1}(I) \parallel \bigotimes_{\cS_t \cap I \neq \emptyset} Q^{+1}(I)\right) = \sum_{\cS_t \cap I \neq \emptyset} D(Q^{-1}(I) \parallel Q^{+1}(I)).
    \end{align*}
    But defining $P^-_s$ and $P^+_s$ to be the distributions of $Z_s$ given $B=1$ and $B=-1$, the data processing inequality implies
    \begin{align*}
       D(Q^-(I) \parallel Q^+(I)) \leq D(P^-(I) \parallel P^+(I)) = 4\gamma^2 w_{I}^2,
    \end{align*}
    where the last equality follows from the formula for the KL divergence between two Gaussians, as we state in Appendix~\ref{sec:omitted}. 
    Hence,
    \begin{align*}
        D\left(\bigotimes_{\cS_t \cap I \neq \emptyset} Q^{-1}(I) \parallel \bigotimes_{\cS_t \cap I \neq \emptyset} Q^{+1}(I)\right) \leq 4\gamma^2\sum_{\cS_t \cap I \neq \emptyset} w_{I}^2.
    \end{align*}

Therefore, due to Pinsker's inequality, the total variation distance between $\bigotimes_{\cS_t \cap I \neq \emptyset} Q^{-1}(I)$ and $\bigotimes_{\cS_t \cap I \neq \emptyset} Q^{+1}(I)$ is at most 
\begin{align*}
    \text{TV}\Big(\bigotimes_{\cS_t \cap I \neq \emptyset} Q^{-1}(I) \parallel \bigotimes_{\cS_t \cap I \neq \emptyset} Q^{+1}(I)\Big) \leq 2\gamma\sqrt{ \sum_{\cS_t \cap I \neq \emptyset} w_{I}^2} \leq 2\gamma.
\end{align*}
where the last equality follows from feasibility of $w$ for $\ILB(\cS)$. 
Then the regret incured by $\cA$ is $(X_t - \frac{1}{2})(x_t - x^*_t)$ with expectation $\epsilon_t |x_t - x^*_t|$. But by  Le Cam's method, since $x_t$ is only a function of $Y_{I}$,
\begin{align*}
    \E[|x_t - x^*_t|] &= (1/2)\E[x_t \mid B = 1] + (1/2)\E[1 - x_t \mid B = -1] \\
    &= (1/2) + (1/2)(\E[x_t \mid B = 1] - \E[x_t \mid B = -1])  \\
    &\geq (1/2) - (1/2)\text{TV}\Big(\bigotimes_{\cS_t \cap I \neq \emptyset} Q^{-1}(I) \parallel \bigotimes_{\cS_t \cap I \neq \emptyset} Q^{+1}(I)\Big) \geq (1/2) - \gamma.
\end{align*}
Therefore, picking $\gamma = 1/4$, we get $\E[|x_t - x^*_t|] \geq 1/4$, which implies the expected regret at round $t$ is at least $\E[\epsilon_t |x_t - x^*_t|] \geq \epsilon_t/4$. Summing over $t$ and using Lemma~\ref{lem:epstlowerbound}, we get 
\begin{align*}
    \Reg(\cA) \geq \gamma\frac{\sum_{t=1}^T\sqrt{ \sum_{\cS_t \cap I \neq \emptyset} w_{I}^2}}{(4\sqrt{2\pi})} = \frac{\sum_{t=1}^T\sqrt{ \sum_{\cS_t \cap I \neq \emptyset} w_{I}^2}}{16\sqrt{2\pi}} \geq \ILB(\cS)/50.
\end{align*}

\end{proof}

\end{document}